\documentclass[letterpaper,10pt,conference]{ieeeconf}

\IEEEoverridecommandlockouts
\usepackage{amsmath, amssymb}
\usepackage{matematica, units}
\usepackage{color, psfrag}
\usepackage{graphicx}
\usepackage{balance, cite}
\usepackage{centernot}
\usepackage{algorithm, algpseudocode}

\newtheorem{assumption}{Assumption}

\newtheorem{lemma}{Lemma}

\newtheorem{remark}{Remark}

\usepackage[unicode=true, bookmarks=true, bookmarksnumbered=true, bookmarksopen=true, bookmarksopenlevel=1, breaklinks=true, pdfborder={0 0 0}, backref=page, colorlinks=false, pagebackref=false]{hyperref}
\hypersetup{pdftitle={}, pdfauthor={Gianluca Garofalo}, pdfsubject={}, pdfkeywords={}, pdfstartview={FitH}}

\title{\LARGE \bf
	Compliant Hierarchical Control for Arbitrary\\
	Equality and Inequality Tasks with Strict and Soft Priorities
}

\author{Gianluca Garofalo
\thanks{Gianluca Garofalo is with ABB Corporate Research, 72178 Västerås, Sweden.
        {\tt\small gianluca.garofalo(at)se.abb.com}}%
}

\begin{document}

\maketitle
\thispagestyle{empty}
\pagestyle{empty}

\begin{abstract}
When a robotic system is redundant with respect to a given task, the remaining degrees of freedom can be used to satisfy additional objectives.
With current robotic systems having more and more degrees of freedom, this can lead to an entire hierarchy of tasks that need to be solved according to given priorities.
In this paper, the first compliant control strategy is presented that allows to consider an arbitrary number of equality and inequality tasks, while still preserving the natural inertia of the robot.
The approach is therefore a generalization of a passivity-based controller to the case of an arbitrary number of equality and inequality tasks.
The key idea of the method is to use a Weighted Hierarchical Quadratic Problem to extract the set of active tasks and use the latter to perform a coordinate transformation that inertially decouples the tasks.
Thereby unifying the line of research focusing on optimization-based and passivity-based multi-task controllers.
The method is validated in simulation.
\end{abstract}

\section{Introduction}
\label{sec:intro}%

The development of complex robotic systems, such as autonomous mobile manipulators and humanoid robots, has led to the design of several control strategies to simultaneously accommodate multiple task-space objectives.
These multiple objectives can be taken into account by the controller because of the kinematic redundancy of the system with respect to a single task-space objective.
By repeatedly applying this concept, one can obtain an entire hierarchy of tasks that need to be solved according to given priorities.
Ordering the task by priority levels ensures that possible conflicts between the tasks can be resolved by relying on their relative importance.

Several works can be found in the literature that address the problem of multi-task control.
A first classification can be made based on the allowed interference between tasks.
Soft-priority methods consider a weighted combination of the control actions required for the individual tasks \cite{Salini2011, Moro2013, Dehio2015, Bouyarmane2018, Englsberger2020}, while strict-priority methods guarantee that there is no interference on a task from any of those with a lower priority level \cite{Siciliano1991, Escande2014, Ott2015, Garofalo2020}.
Another classification criteria is whether the method is based on optimization \cite{Kanoun2011, Escande2014, Herzog2016, Agravante2017, Fahmi2019} or if it computes the control law in closed form \cite{Ott2015, Englsberger2020, Garofalo2020, Wu2023}.
The latter methods are applied to cases with equality tasks only, but have the advantage of having well-analyzed stability properties.
Finally, there are methods based on the well-known Operational Space Formulation \cite{Khatib1987} and other that rely on passivity-based compliance control \cite{Ott2008}.
The first class of methods fully decouples the closed-loop dynamics via feedback-linearization-based actions \cite{Sentis2005, Righetti2011, DelPrete2015}, while the second family of controllers aims at avoiding any reshaping of the individual task-space inertias \cite{Ott2015, Dehio2019, Garofalo2020, Wu2023}.
Feedback linearization typically comes with an easier design phase and leads to decoupled closed-loop linear systems that are also easier to analyze.
Nevertheless, the controller might be more sensitive to modeling errors and parameter variations.
This is exactly the opposite of what happens when the inertia of the system is not reshaped and only recently arbitrary task levels and dimensions could be considered \cite{Garofalo2020}.

Despite these differences, often modifications of a method allow to switch from one class to the other.
Optimization-based methods can easily consider both soft and strict priorities by modifying the cost function and the constraints used by the solver.
This can be similarly achieved via appropriate selection of the damping factor in a damped pseudo-inverses, if only equality tasks are considered.
Also, it is well known that passivity-based methods typically specialize into the operational-space counterpart if one allows to reshape the inertia of each task.
Nevertheless, the possibility of having inequality tasks was so far a prerogative of optimization-based methods with inertia reshaping.

The main contribution of this work is the derivation of a passivity-based hierarchical controller with arbitrary equality and inequality tasks of any dimensions.
In other words, there are no restrictions on how many tasks the users can provide, on whether or not each task is represented via an equality or an inequality and finally on what is the dimension and rank of each task.
The key idea of the method is to extract the set of active tasks computed during the numerical solution of a Weighted Hierarchical Quadratic Problem (WHQP) and utilize them to compute a coordinate transformation needed to design passivity-based hierarchical controllers.
This has been achieved via the following steps:
\begin{itemize}
\item Modification of the Hierarchical Quadratic Problem (HQP) \cite{Escande2014} via weighting matrices (this ensures the inertia decoupling at a later step and allows to have soft priorities between tasks within the same level);
\item Differentiation of the Complete Orthogonal Decomposition (COD);
\item Selection of a coordinate transformation via the set of active tasks and the COD (and its differentiation);
\item Design of the control law in the transformed space.
\end{itemize}

The paper is organized as follows.
Section~\ref{sec:pre} recalls useful definitions and introduces the notation.
The newly proposed WHQP is presented in Section~\ref{sec:hqp}.
Section~\ref{sec:model} describes the considered systems and defines the control objective.
The control design is carried out in Section~\ref{sec:control} and validated in simulations in Section~\ref{sec:validation}.
Finally, Section~\ref{sec:conclusion} summarizes the work.
While these sections contain the main results, all the technical derivations are collected in the Appendices.

\section{Preliminaries}
\label{sec:pre}%

\subsection{Weighted Moore-Penrose Inverse}
Given a matrix ${ A \in \mathbb{R}^{m \times n} }$ and two symmetric positive definite matrices ${ W_1 \in \mathbb{R}^{m \times m} }$ and ${ W_0 \in \mathbb{R}^{n \times n} }$, the following problem for an unknown matrix ${ X \in \mathbb{R}^{n \times m} }$:
\begin{subequations}
	\begin{align}
		& A X A = A \\
		& X A X = X \\
		& (W_1 A X)^\top = W_1 A X \\
		& (W_0 X A)^\top = W_0 X A
	\end{align}%
	\label{eq:def_generalized_inverse}%
\end{subequations}%
has a unique solution ${ X = A^{\dagger_{_{W_1,W_0}}} }$ called the Weighted Moore–Penrose Inverse (WMPI) of $A$ \cite{Ben-Israel, Doty1993}.
Notice that $A$ is not assumed to have full rank and therefore ${ A^{\dagger_{_{W_1,W_0}}} }$ cannot be computed as a left or right weighted inverse.
Instead, one can compute ${ A^{\dagger_{_{W_1,W_0}}} }$ as
\begin{align}
	A^{\dagger_{_{W_1,W_0}}} = R_0^{-1} \hat A^\dagger R_1
\end{align}
where $\hat A^\dagger$ is the classic Moore–Penrose Inverse (MPI) of ${ \hat A = R_1 A R_0^{-1} }$ and the Cholesky decomposition of the weight matrices are ${ W_1 = R_1^\top R_1 }$ and ${ W_0 = R_0^\top R_0 }$ \cite{Garofalo2020}.
In turn, $\hat A^\dagger$ can be conveniently computed from the compact\footnote{A Complete Orthogonal Decomposition (COD) will be referred to as compact if it considers only the non-zero blocks of the original COD.} Complete Orthogonal Decomposition (COD) of ${ \hat A = U L Y^\top }$ as ${ \hat A^\dagger = Y L^{-1} U^{\top} }$.

\subsection{Notation}
A simplifying notation is introduce to denote stacks of matrices or vectors.
The whole stack of vectors ${ x_k \in \mathbb{R}^{m_k} }$ for ${ k \in \{1,\ldots,r\} }$ will be denoted by suppressing the subscript and organized in a column array, i.e., ${ x = \begin{bmatrix} x_1^\top & \ldots & x_r^\top \end{bmatrix}^\top \in \mathbb{R}^m }$, with ${ m = \sum_{k=1}^r m_k }$.
Moreover, the partial stack of the first $k$-th elements will be denoted as $\mathbf x_k$.
Therefore, ${ \mathbf x_1 = x_1 }$ and ${ \mathbf x_r = x }$.

The simplified notation ${ A_k^+ }$ will be used in place of ${ A_k^{\dagger_{_{W_k,M}}} }$ with $M$ being the joint inertia matrix of the robot and $W_k$ the weight matrix used at level $k$.
Moreover, for a stack of matrices $A$ it is easy to show that ${ A^+ = \begin{bmatrix} A_1^+ & \ldots & A_r^+ \end{bmatrix} }$ is the WMPI of $A$ with weight matrices $W$ and $M$, being $W$ the block-diagonal matrix whose blocks are $W_k$.
This is easily verified by showing that $A^+$ so defined satisfies \eqref{eq:def_generalized_inverse}.

Finally, $E$ and $O$ will always denote the identity matrix and a matrix of zeros, with dimensions given by the context.

\subsection{Projected Stack}
Given a stack of matrices ${ A \in \mathbb{R}^{m \times n} }$, the dynamically consistent stack ${ \bar A }$ is computed via the recursion
\begin{align}
	& \bar A_k = A_k \, P_{k-1}
	\label{eq:GSP-1}%
	\\
	& P_k = P_{k-1} \, \br[r]{ E - \bar A_k^+ \bar A_k }
	\label{eq:GSP-2}%
\end{align}%
for ${ k \in \br[c]{1,\ldots,r} }$ and with $P_0 = E$.
A consequence of the definition is that the first matrix in the stack is unchanged.
Moreover, the matrices $P_k$ and $\bar A_k$ satisfy the properties
\begin{align}
	& P_k \, P_k = P_k
	\label{eq:PP}%
	\\
	& P_k^\top \, M = M \, P_k
	\label{eq:PT}%
	\\
	& P_k \, P_j = P_j \, P_k = P_k
	\label{eq:P3}%
	\\
	& A_j \, \bar A_k^+ = 0
	\label{eq:bJpbJ}%
\end{align}%
for all ${ j < k }$.
The first two qualify $P_k$ as an orthogonal projection matrix with metric $M$ and \eqref{eq:P3} additionally guarantees that each matrix in the stack is projected in the nullspace of all the previous levels.
Due to the recursive nature of the definitions, these properties were proved in \cite{Garofalo2020} by induction.

\begin{remark}
The iterations \eqref{eq:GSP-1}, \eqref{eq:GSP-2} are a generalization of the Gram–Schmidt process used to compute the LQ decomposition of a matrix without normalization of the factor $Q$.
Indeed, ${ A \bar A^+ }$ is a block lower-triangular due to \eqref{eq:bJpbJ}.
\end{remark}

\section{Weighted Hierarchical Quadratic Problem}
\label{sec:hqp}%

The goal of this section is twofold: recalling the results in \cite{Escande2014} using the notation of this paper, while modifying them to allow a weighted version of the problem.
In particular, while in \cite{Escande2014} a MPI was considered for each of the matrix appearing in the constraints of the problem, here a projected stack via WMPI will be created from the matrices.
Therefore, it will be referred to as Weighted Hierarchical Quadratic Problem (WHQP).
This has two main consequences:
\begin{itemize}
\item soft priorities are possible between tasks within the same priority level,
\item an inertially decoupled coordinate transformation can be identified in Section~\ref{sec:control}.
\end{itemize}

The Hierarchical Quadratic Problem (HQP) introduced in \cite{Escande2014} solves efficiently a sequence of QP problems that incorporate equality or inequality constraints at any levels, while respecting the task priorities.
The sequence of optimal objectives is minimal in the lexicographic sense: it is not possible to decrease an objective without increasing all those with higher priority.
In the WHQP, each objective is weighted via the symmetric positive-definite matrix $W_k$, i.e., 
\begin{subequations}
\begin{align}
	& \underset{x, w_1, \ldots, w_r}{\textrm{lex min}}	& & \hspace*{-1.2cm} \br[c]{ \frac{1}{2} w_1^\top W_1 w_1, \ldots, \frac{1}{2} w_r^\top W_r w_r } \\
	& \text{ s.t.}			& & \hspace*{-1.2cm} A_{k} \, x \geq b_k + w_k
\end{align}%
	\label{eq:hqp}%
\end{subequations}
for ${ k \in \br[c]{1,\ldots,r} }$, where $w_k$ are slack variables that can relax the constraint in case of infeasibility and the notation with only lower bounds encompasses upper bounds, double bounds and equalities.
When the matrix $A_k$ is itself a stacked matrix and $W_k$ is block-diagonal, this formulation allows to have soft priorities between tasks at the same level $k$.
To solve problem \eqref{eq:hqp}, a modification of the hierarchical active search method proposed in \cite{Escande2014} is presented.

As in \cite{Escande2014}, a key step of the algorithm is the solution of the QP for the $k$-th priority level in the form:
\begin{subequations}
\begin{align}
	& \min_{x^{(k)}, w_k}	& & \hspace*{-1.7cm} \frac{1}{2} w_k^\top W_k w_k \\
	& \text{ s.t.}			& & \hspace*{-1.7cm} \mathbf{\bar A}_{k-1} \, x^{(k)} = \mathbf{\bar A}_{k-1} x^{(k-1)} \\
	&						& & \hspace*{-1.7cm} A_k \, x^{(k)} = b_k + w_k
\end{align}%
\label{eq:WHQP_k}%
\end{subequations}
The first constraint guarantees that the solution of the problem at step $k$, denoted by $x^{(k)}$, does not affect the previous tasks at levels prior to $k$.
The optimal solution for the primal and dual variables is found in Appendix \ref{sec:kkt} to be:
\begin{align}
	& x^{(k)} = \mathbf{\bar A}^{+}_k \, \mathbf{v}_k \\
	& w_k = - \big( E - \bar A_k \bar A^+_k \big) v_k \\
	& \lambda_k = W_k w_k \\
	& \mathbf{\lambda}_{k-1} = - \mathbf{\bar A}^{+ \top}_{k-1} A^\top_k W_k w_k
\end{align}%
where ${ v_k = b_k - A_k \mathbf{\bar A}_{k-1} \mathbf v_{k-1} }$ and ${ \mathbf v_{k-1} }$, ${ \mathbf{\lambda}_{k-1} }$ are to be considered void for ${ k = 1 }$.

In presence of inequality constraints, one needs in addition that the corresponding dual variables $\lambda$ are non-positive.
This is guaranteed by Alg.~\ref{alg:has}, so that one obtains the optimal solution of \eqref{eq:hqp}.
The algorithm relies on the solution of a WHQP with only equality constraints every time the active set $\mathcal{A}$ changes.
Both the values of the primal variables $x$ and dual variables $\lambda$ at the optimum are needed, as well as the index $\eta$ of the last task in the stack that contributes to the solution $x$.
The primal and dual variables are returned by the algorithms $eWHQP_{primal}$ and $eWHQP_{dual}$, respectively.
In contrast, in \cite{Escande2014}, the algorithms $eHQP_{primal}$ and $eHQP_{dual}$ were used since the problem was not weighted.
Moreover, in \cite{Escande2014} these algorithms were using the COD to compute the MPI.
Appendix~\ref{sec:compact} shows how the WMPI needed for the optimal solution can be computed using the COD iterations in the exact analogous way as in \cite{Escande2014}.

Exactly as in \cite{Escande2014}, Alg.~\ref{alg:has} also keeps track of the locked set $\mathcal{L}$ (and its complement $\mathcal{U}$), i.e., the set of inequalities that cannot be removed from $\mathcal A$ because are essential to obtain the minimum value of objectives with higher priorities.
Finally, $\mathcal{E}$ and $\mathcal{I}$ are the constant sets of equality and inequality tasks.
\begin{algorithm}
\caption{Weighted Hierarchical Active Search}
\label{alg:has}%
\begin{algorithmic}[1]
\State $\mathcal{A} \gets \mathcal{E}$
\State $\mathcal{L} \gets \emptyset$
\State $x \gets 0$
\While{$h \leq r$}
	\Repeat
	\State $(x, \eta) \gets$ \Call{$eWHQP_{primal}$}{$\mathcal{A}$}
	\ForAll{$k \in \mathcal{I} : k \leq \eta$ and $A_k x < b_k$}
		\State $\mathcal{A} \gets \mathcal{A} \cup \{k\}$
	\EndFor
	\Until{not $\mathcal{A}$ is new}
	\Repeat
		\State $\mathcal{W} \gets \mathcal{A} \cap \mathcal{I} \cap \mathcal{U}$
		\State $\lambda \gets$ \Call{$eWHQP_{dual}$}{$h, \mathcal{W}$}
		\ForAll{$k \in \mathcal{W} : k \leq \eta$}
			\If{$\lambda_k > 0$}
				\State $\mathcal{L} \gets \mathcal{L} \cup \{k\}$
			\Else
				\State $\mathcal{A} \gets \mathcal{A} \setminus \{k\}$
			\EndIf
		\EndFor
	\Until{$h \leq r$ and $\mathcal{A}$ is new}
\EndWhile
\State \Return $x$
\end{algorithmic}
\end{algorithm}

The reader is referred to \cite{Escande2014} for the optimality and convergence analysis of the algorithm, as well as all for further details, because it will be shown in the next subsection that with minor changes the Weighted Hierarchical Active Search can be obtained from the unweighted counterpart in \cite{Escande2014}.

\subsection{Efficient Computation via COD}
Appendix~\ref{sec:compact} shows how to use the COD and Cholesky decomposition to compute each matrix of the projected stack and corresponding WMPI needed to obtain the solution of Alg.~\ref{alg:has}.
From the expressions in Appendix~\ref{sec:compact}, it is further quite easy to realize that only two modifications need to be made in order to compute the solution of a WHQP compared to the original HQP.
Let $R_0$, $R_k$ be the Cholesky factor of $M$ and $W_k$ respectively, i.e., ${M = R_0^\top R_0}$ and ${W_k = R_k^\top R_k}$ and recall that $Z_k$ is computed from the factors of each COD, see \cite{Escande2014} and Appendix~\ref{sec:compact}.
Then firstly, $Z_0$ needs to be initialized to $R_0^{-1}$ instead of the identity as in \cite{Escande2014}.
Secondly, each $A_k$, $b_k$ need to be pre-multiplying by the corresponding Cholesky factor $R_k$.
With these two modifications in place, the implementation of the WHQP and the regular HQP is exactly identical.
More specifically, the variables ${ R_k A_k }$, ${ R_k b_k }$, ${ R_k w_k }$, ${ R_k v_k }$ and ${ R_k^{-\top} \lambda_k }$ in this paper correspond to $A_k$, $b_k$, $w_k$, $v_k$ and $\lambda_k$ appearing in the algorithms in \cite{Escande2014}.
This correspondence will be assumed for the remainder of this paper.
With an abuse of notation following this correspondence, it is shown in Appendix~\ref{sec:compact} that the solution requires the computation of the WMPI as
\begin{align}
& \bar A_k^+ = Z_{k-1} Y_k L_k^{-1} U_k^\top
\end{align}
where $U_k$, $L_k$ and $Y_k$ are the factors of the compact COD of $A_k Z_{k-1}$, with
\begin{subequations}
\begin{align}
	& Z_k = Z_{k-1} \tilde Z_k \\
	& Z_0 = R_0^{-1}
\end{align}%
\label{eq:Z_recursion}%
\end{subequations}
for ${ k \in \br[c]{1,\ldots,r} }$, $\tilde Z_k$ the complement of $Y_k$ in the COD, i.e, ${Y_k^\top \tilde Z_k = 0}$, while ${Z_k^\top M Z_k = E}$.

\subsection{Summary}
The WHQP generalizes the regular HQP, in the same sense that the WMPI generalizes the MPI.
This is exactly because the WHQP is solved using a WMPI in place of a simple MPI.
The benefit is that both weight matrices in the WMPI can be used for a specific purpose.
By setting them to $W_k$ and $M$, the first weight $W_k$ is used to have soft priorities between tasks within the same priority level $k$, while the second will result in an inertially-decoupled coordinates transformation in Section~\ref{sec:control}.
From an algorithmic point of view, the solution of a WHQP can be obtained with the exact same implementation used to solve a regular HQP.
Only two minor modifications are necessary: first, $Z_0$ needs to be initialized to $R_0^{-1}$ instead of the identity matrix; second, each $A_k$ and $b_k$ need to be pre-multiplied by $R_k$.

\section{Robot Model and Control Objective}
\label{sec:model}%

The considered fully-actuated robotic system is modeled by the nonlinear differential equations:
\begin{align}
	M(q) \ddot q + C(q, \dot q) \dot q + g(q) = \tau \ ,
	\label{eq:model}%
\end{align}%
where the state of the robot is given by generalized positions and velocities $q$, ${ \dot q \in \mathbb{R}^n }$, $n$ being the number of degrees of freedom (DoF).
The dynamic matrices are the symmetric and positive definite inertia matrix ${ M \in \mathbb{R}^{n \times n} }$, a Coriolis matrix ${ C \in \mathbb{R}^{n \times n} }$ satisfying the passivity property ${ \dot M = C + C^\top }$ and the gravity torque vector ${ g \in \mathbb{R}^n }$.
Finally, the control input ${ \tau \in \mathbb{R}^n }$ is realized through the motors of the robot.
Due to the use of rotational joints or prismatic joints with end-stops, the configuration space of the robot is bounded.
This implies the boundedness of the eigenvalues of $M$.

\subsection{Control objective}
Assume that the user provides a finite number of tasks and a relative level of strict priority among them.
Tasks at a same level of priority can further have a different weighting to specify their relative importance within the same priority level.
The tasks at level $k$ are expressed in coordinates as
\begin{subequations}
	\begin{align}
		& x_i = f_k(q) \\
		& \nu_k = J_k(q) \, \dot q \ ,
	\end{align}%
\end{subequations}%
for ${ k \in \br[c]{1,\ldots,r} }$, ${ f_k : \mathbb{R}^n \rightarrow \mathbb{R}^{m_k} }$ and symmetric positive-definite block-diagonal weighting matrix ${ W_k \in \mathbb{R}^{m_k \times m_k} }$.
The dimensions of the blocks of $W_k$ are the ones of the tasks with the same priority level $k$.
The tasks in each priority level are organized in a priority stack such that $x_k$ has a higher priority than $x_j$, for all ${ k<j }$.
The total task dimension $m$ is given by ${ m = \sum_{k=1}^r m_k }$.
According to the notation of this paper, the whole stack is ${ x = \begin{bmatrix} x_1^\top & \ldots & x_r^\top \end{bmatrix}^\top \in \mathbb{R}^m }$.
\begin{assumption}
\label{th:more_tasks}%
\textit{The total task dimension is not less than the DoF of the robot, i.e., ${ m \geq n }$.
Moreover, the projected stack $\bar J$ is always full rank, i.e., ${ \rank(\bar J) = n }$.}
\end{assumption}%
This means that tasks cannot be all singular at the same time.
The previous assumption can be easily satisfied by making sure that the stack always contains a configuration task, i.e., a task whose Jacobian matrix is the identity.
For example, by adding it at the end of the stack.

For the $k$-th task a desired trajectory ${ x^d_{k}(t) \in \mathbb{R}^{m_k} }$ is provided, together with the desired velocity coordinates and derivative, i.e., $\nu^d_{k}(t)$ and $\dot \nu^d_{k}(t)$.
As for the task coordinates, the stack of desired trajectories will be denoted by omitting the subscript, e.g., ${ x^d \in \mathbb{R}^m }$.
An equality task is fulfilled if ${x_k = x^d_k}$, while, without loss of generality, an inequality task is fulfilled if ${x_k > x^d_k}$.
If the task involves only velocity, the previous relationship for position coordinates are void.

Intuitively, the goal of the controller is to let the robot follow a trajectory in state space that results in task-space trajectories that are as close as possible to the desired ones, while respecting the priority levels of the tasks.

\section{Compliant Hierarchical Control Laws}
\label{sec:control}%

A first-order kinematic solution of the problem formulated in the previous section, is given by solving the WHQP
\begin{subequations}
\begin{align}
	& \underset{\dot q, w_1, \ldots, w_r}{\textrm{lex min}}	& & \hspace*{-0.7cm} \br[c]{ \frac{1}{2} w_1^\top W_1 w_1, \ldots, \frac{1}{2} w_r^\top W_r w_r } \\
	& \text{ s.t.}			& & \hspace*{-0.7cm} J_{k} \, \dot q \geq x_k^d + w_k
\end{align}%
	\label{eq:control_WHQP}%
\end{subequations}
but it is unclear at this stage how this can be used to compute a control torque $\tau$ leading to a compliant closed-loop system.

The design of the control law hinges on the computation of a change of coordinates obtained via a careful examination of the solution of \eqref{eq:control_WHQP}.
In fact, since in Alg. \ref{alg:has} the optimal solution is given by the active set only, then one can write
\begin{align}
	\dot q = \sum_{k \in \mathcal{A}} \bar A_k^+ v_k = \sum_{k \in \mathcal{A}} Z_{k-1} Y_k L_k^{-1} U_k^\top v_k
\end{align}%
where only the WMPI of the active tasks and corresponding $v_k$ play a role.
Moreover, it is worth to remember that $U_k$ is a base of the range space for the active tasks at level $k$.
This suggests that the mapping ${ \xi = F \dot q }$, with \begin{subequations}
\begin{align}
	& F^{-1} = \begin{bmatrix} \cdots & Z_{k-1} Y_k & \cdots \end{bmatrix} \\
	& F = F^{-\top} M
\end{align}%
\end{subequations}
for $k \in \mathcal{A}$, can be used to express the dynamic model of the robot in the active-tasks space and then design passivity-based compliant controllers in the new coordinates.
Using the new velocity coordinates $\xi$ and the new control input ${ f = F^{-\top} \big( \tau - g \big) }$ the model \eqref{eq:model} can be rewritten as
\begin{subequations}
\begin{align}
	& \dot q = F^{-1} \xi \\
	& \dot \xi + \Gamma \xi = f
\end{align}%
\end{subequations}
where notably the transformed inertia matrix is the identity, thanks to the definition of $F$ i.e., ${ F^{-\top} M F^{-1} = F F^{-1} = E }$, and the transformed Coriolis matrix
\begin{align}
	\Gamma = F^{-\top} \Big( C F^{-1} + M \der{}{t}(F^{-1}) \Big)
\end{align}
is skew-symmetric, in view of the passivity property \cite{Ott2008} of the model \eqref{eq:model}.
Clearly such a coordinates transformation is possible if and only if one can compute the time derivative of $F^{-1}$ to obtain $\Gamma$.
Appendix \ref{sec:cod} shows how this can be achieved by proposing an algorithm to differentiate the compact COD.
Notice that both $\xi$ and $f$ can be thought of as stacked vectors with each block corresponding to one of the active tasks.
Finally, Assumption \ref{th:more_tasks} guarantees that $F$ is always square and invertible.
However, the number of active tasks and the dimension of each of them is not fixed, rather it is automatically adjusted by the WHQP given the rank of each task in the stack and the priority levels.
Still, $F^{-1}$ always exists and it is differentiable.

Having the model in the new coordinates, it is easy to design a compliant control law as follows:
\begin{align}
	& \tau = g + F^\top \big( \dot \xi^r + \Gamma_d \xi^r + \Gamma_s \xi - D (\xi - \xi^r) \big)
	\label{eq:Slotine}%
\end{align}%
where we have a combination of feedforward, feedback and cancellation terms.
Given the error term ${ \tilde x = x - x^d }$, $\xi^r$ appearing in the control law is computed by stacking
\begin{align}
	& \xi_k^r = L_k^{-1} U_k^\top \big( \nu_k^d - \bar K \tilde x_k \big)
\end{align}%
for $k \in \mathcal{A}$.
Moreover, the matrices $\Gamma_d$ and $\Gamma_s$ are such that ${ \Gamma = \Gamma_d + \Gamma_s }$, with $\Gamma_d$ obtained by extracting the blocks on the diagonal, each of dimension of the corresponding active task.
Therefore, it follows that both $\Gamma_d$ and $\Gamma_s$ are skew-symmetric as well.
Lastly, ${ \bar K \in \mathbb{R}^{m \times m} }$ and ${ D \in \mathbb{R}^{n \times n} }$ are symmetric and positive definite gain matrices.
To avoid coupling between the tasks, these matrices are block-diagonal.

\begin{remark}
It is worth to highlight the connection with the passivity-based compliance control proposed in \cite{Ott2008}, since most of the works on this topic extend the results presented therein.
By redefining $\xi$ and $F^{-1}$ in a way that the factors $L_k^{-1}$ are moved from $\xi$ to $F^{-1}$ (denote them by $\xi^\prime$ and $F^\prime$), one obtains that $\xi^\prime$ contains directly the task velocity $\dot x_1$ if $J_1$ is full rank.
If one creates a stack in which this task is followed by a configuration task, then $F^\prime$ performs the same change of coordinates as ${ J_N }$ in \cite{Ott2008}.
Moreover, the block-diagonal structure of the transformed inertia matrix is given by ${ L_k^{-\top} L_k^{-1} }$.
In other words, while $\xi$ correspond to active tasks momenta, $\xi^\prime$ are active tasks velocities.
\end{remark}

\subsection{Summary and Discussion}
The implementation of the proposed controller can be thought of as separated in two steps: firstly a Weighted Hierarchical Active Search is performed to find the active set, given the stack of tasks that the robot needs to perform; secondly the matrices required to compute the control torques are extracted from the corresponding active set.
The search of the active set provides an automatic way to obtain task coordinates consistent with the priority levels.
The proposed controller is also robust to singularity because singular directions will be removed from the active set in favor of the next candidates in the stack.
This happens naturally and automatically since when a task becomes singular, it is reflected directly in the factors of the COD.

The term multiplying $F^\top$ in \eqref{eq:Slotine} is the new control input $f$ and therefore a generalization of the Slotine-Li controller \cite{Slotine}, except that special care is required for the position coordinates.
As a sketch, one can think of the configuration manifold as an embedded Riemannian submanifold in the higher dimensional linear space given by all the tasks.
Keeping in mind that the Riemannian gradient of a function is given by projecting a smooth extension defined in the embedding space onto the tangent space of the embedded submanifold \cite[Chapter 3]{Boumal}, one can connect the terms in the control law to projections in the space of active tasks to show local stability of the closed-loop system.
A detailed analysis is at this stage beyond the scope of the work.

Finally, external forces collocated with each task appear in the model upside down, so one need to measure the external interaction if a complete decoupling is wished \cite{Wu2023}.

\section{Validation}
\label{sec:validation}%

The control law is validated in simulation using a Franka Emika Panda robot.
This $7$-DoF manipulator is often used in research and therefore the result of this work could be relatively easily replicated by the interested readers.
Similarly, although a proprietary implementation of the WHQP has been used for the validation of the proposed approach, Section~\ref{sec:hqp} has clearly reported the two minor modifications that need to be performed to the HQP, for which open source implementations are available \cite{Escande2014}.

\begin{figure}
	\centering
	\includegraphics[width=0.7\columnwidth]{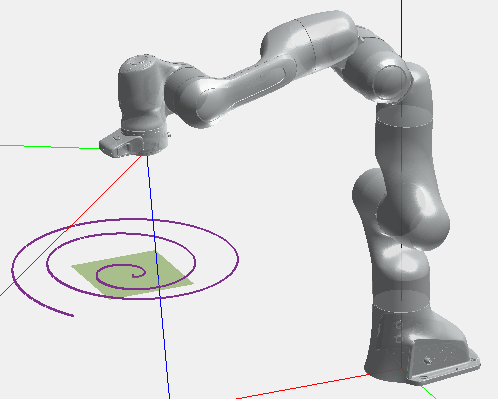}
	\caption{Simulated Franka Emika Panda robot with relevant frames, desired trajectory (purple spiral) and allowed region (green square).}
	\label{fig:intro}
\end{figure}

The robot is required to track the spiraling trajectory in Fig.~\ref{fig:intro} with its TCP, while changing its orientation in a non-trivial way consisting of a rotation around a moving axis (see Fig.~\ref{fig:htm_high}-\ref{fig:htm_avg} for the precise time evolution of the desired rotation matrix).
In addition, the square in Fig.~\ref{fig:intro} represents the admissible space in which the TCP is allowed to stay at all times.
As a result, the robot is only able to track the portion of the trajectory within the square and the scenario is used to exemplify how both equality and inequality tasks can be handled by the proposed approach.
Finally, the tracking of the trajectory for the position is weighted with the requirement of keeping the TCP at the center of the spiral, as an example of soft priorities.
The complete list of tasks in decreasing level of priority is:
\begin{itemize}
\item Priority 1: Orientation tracking.
\item Priority 2: TCP's $x$ and $y$ position coordinates within the square.
\item Priority 3: Tracking of the spiral weighted with regulation towards its center.
\item Priority 4: Keep initial joint configuration.
\end{itemize}
The weighting of the two tasks at level 3 is used to show how the approach can include both strict and soft priorities.
In particular, Fig.~\ref{fig:htm_high} and Fig.~\ref{fig:htm_low} show that with very skewed weights the resulting behavior tends towards the one obtained with strict priorities, while Fig.~\ref{fig:htm_avg} shows a compromise between the two tasks when the weights are equal.

\begin{figure}
	\includegraphics{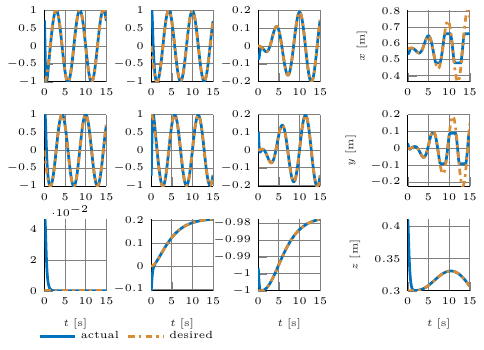}
	\caption{Non-trivial part of the homogeneous transformation matrices for the case of tracking weight much bigger than the regulation weight.
	The desired orientation in orange is perfectly tracked (actual values in blue), since the orientation is the first task in the stack.
	The position is tracked as long as the desired trajectory is within the admisable box.}
	\label{fig:htm_high}%
\end{figure}

\begin{figure}
	\centering
	\includegraphics[width=0.99\columnwidth]{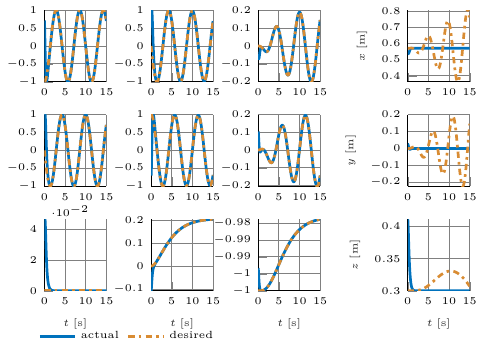}
	\caption{Non-trivial part of the homogeneous transformation matrices for the case of tracking weight much smaller than the regulation weight.
	While there is no change in the orientation, since it is the first task in the stack, the position of the TCP now converges to the center of the spiral.}
	\label{fig:htm_low}%
\end{figure}

\begin{figure}
	\centering
	\includegraphics[width=0.99\columnwidth]{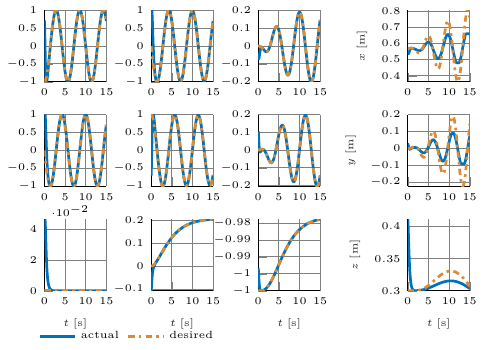}
	\caption{Non-trivial part of the homogeneous transformation matrices for the case of tracking weight equal to the regulation weight.
	Also in this case, there is no change for the orientation.
	Nevertheless, the position of the TCP is in between the tracking of the spiral and reaching its center.}
	\label{fig:htm_avg}%
\end{figure}

\paragraph*{Priority 1} The desired velocity is the desired angular velocity $\omega_d$, while the configuration error is given by the vector part of the quaternion error extracted from the actual and desired orientation matrices.
This corresponds to a reference term given by ${ \omega^r = \omega_d - K_1 e_o }$, where ${ K_1 = 12 E_{3 \times 3} }$ in SI units.
The motivation is that ${ \omega = \omega^r }$ guarantees that the desired orientation is tracked \cite{Siciliano}.

\paragraph*{Priority 2} This inequality tasks are implemented using reference terms of the type ${ p_i^r = - k_2 \big( p_i - b_i \big) }$, where each $p_i$ is a coordinate of the TCP in the horizontal plane, $b_i$ the corresponding box bound and ${ k_2 = 10 }$ in SI units.
Therefore, ${ K_2 = 10 E_{4 \times 4} }$.
The motivation is that ${ \dot p \geq p^r }$ guarantees that $p$ stays within the bounds $b$ \cite{Ames2019}.

\paragraph*{Priority 3} Given the desired spiral values $s^d$ and $\dot s^d$, with center at $c^d$, the reference term is obtained stacking ${ \dot s^d - K_3^\prime \big( p - s^d \big) }$ and ${ - K_3^\prime \big( p - c^d \big) }$, with ${ K_3^\prime = 5 E_{3 \times 3} }$ in SI units and ${ K_3 = \diag(K_3^\prime, K_3^\prime) }$.
The values of $W_3$ used in the 3 cases are ${ \diag(E_{3 \times 3}, 10^{-6} E_{3 \times 3}) }$, ${ \diag(10^{-6} E_{3 \times 3}, E_{3 \times 3}) }$ and ${ E_{6 \times 6} }$ respectively.

\paragraph*{Priority 3} The reference term is ${ - K_4 \big( q - q_0 \big) }$, with ${ q_0 }$ the initial configuration and ${ K_4 = 14 E_{7 \times 7} }$ in SI units.

Finally, $\bar K$ is the block-diagonal matrix obtained from the blocks $K_k$, ${ k \in \br[c]{1,2,3,4} }$, while $D$ is obtained by extracting the blocks corresponding to the active tasks from $\bar D$, which was set to be numerically equal to $\bar K$.

\section{Conclusion}
\label{sec:conclusion}%

Optimization-based hierarchical control has proven useful to automatically address multiple tasks with given priorities levels and to include both equality and inequality tasks.
At the same time, passivity-based controllers are very popular to create a compliant behavior of the system without requiring measurements of the external torques.
Nevertheless, since the design of passivity-based controllers uses a diffeomorphism between joint-space velocities and task velocities, restrictive assumptions are usually made to combine hierarchical control and passivity-based control, e.g., only equality tasks are allowed.
This work has shown that these assumptions can be lifted by using the COD of the active task set of a Weighted Hierarchical Quadratic Problem, which finds an optimal solution given all the tasks and corresponding strict and soft priorities.
The result is an algorithm that allows the user to freely provide as many equality and inequality tasks as required, without any particular attention to guarantee a priory that the tasks will not become singular during execution.
The occurrence of singularities is automatically handled in the Weighted Hierarchical Active Search by using the DoFs released by a singular task to execute the next task in the stack according to the priority levels.

\appendices
\section{Optimality conditions}
\label{sec:kkt}%

The optimal solution of \eqref{eq:WHQP_k} can be computed using the Lagrangian method.
This leads to the system of equations:
\begin{align}
	& W_k w_k = \lambda_k \\
	& \mathbf{\bar A}^{\top}_{k-1} \, \mathbf{\lambda}_{k-1} + A^\top_k \, \lambda_k = 0 \\
	& \mathbf{\bar A}_{k-1} \, x^{(k)} = \mathbf{\bar A}_{k-1} x^{(k-1)} \\
	& A_k \, x^{(k)} = b_k + w_k
\end{align}%
Assume as induction step that ${ x^{(k-1)} = \mathbf{\bar A}^{+}_{k-1} \, \mathbf{v}_{k-1} }$, with ${ \mathbf{v}_1 = b_1 }$, then it follows that to keep satisfying the previous constraints, the solution at the current step is written as
\begin{align}
	& x^{(k)} = \mathbf{\bar A}^{+}_{k-1} \, \mathbf{v}_{k-1} + P_{k-1} x_p
\end{align}%
where $x_p$ can be freely chosen to best satisfy the additional constraints of the $k$-th level.
That is
\begin{align}
	& A_k \, \mathbf{\bar A}^{+}_{k-1} \, \mathbf{v}_{k-1} + \bar A_k x_p = b_k + w_k
\end{align}%
and therefore defining ${ v_k = b_k - A_k \, \mathbf{\bar A}^{+}_{k-1} \, \mathbf{v}_{k-1} }$ and setting ${ x_p = \bar A^+_k \, v_k }$, the primal solutions are
\begin{align}
	& x^{(k)} = \mathbf{\bar A}^{+}_k \, \mathbf{v}_k \\
	& w_k = - \big( E - \bar A_k \bar A^+_k \big) v_k
\end{align}%
where $w_k$ is computed from $ - w_k = b_k - A_k x^{(k)} = b_k - A_k \mathbf{\bar A}^{+}_{k-1} \, \mathbf{v}_{k-1} - \bar A_k \bar A^+_k \, v_k = v_k - \bar A_k \bar A^+_k \, v_k $.

Last the dual variables $\mathbf{\lambda}_{k-1}$ need to be computed by solving ${ \mathbf{\bar A}^{\top}_{k-1} \, \mathbf{\lambda}_{k-1} = - A^\top_k w_k }$, leading to
\begin{align}
	& \mathbf{\lambda}_{k-1} = - \mathbf{\bar A}^{+ \top}_{k-1} A^\top_k \lambda_k
\end{align}%
which can be proved to be the solution by checking whether ${ \mathbf{\bar A}^{\top}_{k-1} \mathbf{\bar A}^{+ \top}_{k-1} A^\top_k = A^\top_k }$.
This is indeed the case because
\begin{align}
	A_k \mathbf{\bar A}^{+}_{k-1} \mathbf{\bar A}_{k-1} &= A_k \Big( E - \sum_{i=k}^r \bar A^+_i \bar A_i \Big) = A_k
\end{align}%
since ${ A_k \begin{bmatrix} \bar A^+_k & \ldots & \bar A^+_r \end{bmatrix} = 0 }$ due to \eqref{eq:bJpbJ}, see \cite{Garofalo2020}.

\section{Projected Stack via COD}
\label{sec:compact}%

\begin{lemma}
The WMPI $A^{\dagger_{_{W_1,W_0}}}$ can be computed as ${ W_0^{-1} A^\top R_1^\top (R_1 A W_0^{-1} A^\top R_1^\top )^\dagger R_1^\top }$.
\end{lemma}
\begin{proof}
The result follows directly from the identity ${ A^\dagger = A^\top (A A^\top)^\dagger }$ and ${ A^{\dagger_{_{W_1,W_0}}} = R_0^{-1} \hat A^\dagger R_1^\top }$ with ${ \hat A = R_1 A R_0^{-1} }$.
\end{proof}

Given the previous lemma, one can show that each of the nullspace projectors $P_k$ can be expressed as
\begin{align}
	P_k = Z_k Z^\top_k M
\end{align}%
with ${ Z_k^\top M Z_k = E }$ and the advantage that while all $P_k$ are square and with dimension $n$, $Z_k$ has as many columns as the remaining DoFs at level $k$.
Therefore, $Z_k$ will be eventually empty when all the DoFs have been used up by the tasks.
This can be proved by induction, as it is certainly true for $P_0$ with ${ Z_0 = R_0^{-1} }$.
Assume the condition true for $P_k$, then
\begin{align}
	& \bar A_{k+1} = R_{k+1}^{-1} \hat A_{k+1} Z^\top_k M \\
	& \hat A_{k+1} = R_{k+1} A_{k+1} Z_k
\end{align}%
from which
\begin{align}
\begin{split}
	\bar A_{k+1}^+ &= M^{-1} \bar A_{k+1}^\top R_{k+1}^{\top} (R_{k+1} \bar A_{k+1} M^{-1} \bar A_{k+1}^\top R_{k+1}^{\top})^\dagger R_{k+1} \\
	&= Z_k \hat A^\top_{k+1} \big( \hat A_{k+1} \hat A_{k+1}^\top \big)^\dagger R_{k+1} = Z_k \hat A_{k+1}^\dagger R_{k+1}
\end{split}%
\end{align}%
and therefore
\begin{align}
	& N_{k+1} = E - \bar A_{k+1}^+ \bar A_{k+1} = E - Z_k \hat A_{k+1}^\dagger \hat A_{k+1} Z_k^\top M \\
	\begin{split}
	& P_{k+1} = P_k N_{k+1} = Z_k Z_k^\top M - Z_k \hat A_{k+1}^\dagger \hat A_{k+1} \hat Z_k M \\
	& \phantom{P_{k+1}} = Z_k \big( E - \hat A_{k+1}^\dagger \hat A_{k+1} \big) Z_k^\top M
	\end{split}
\end{align}%
The proof is completed by computing the MPI of $\hat A_{k+1}$ via COD, which additionally clarifies how to compute each $Z_k$.
For each matrix, we have
\begin{align}
	& \hat A_k = \begin{bmatrix} V_k & U_k \end{bmatrix} \begin{bmatrix} 0 & 0 \\ L_k & 0 \end{bmatrix} \begin{bmatrix} Y_k & \tilde Z_k \end{bmatrix}^\top = U_k L_k Y_k^\top
\end{align}%
where $\begin{bmatrix} V_k & U_k \end{bmatrix}$ and $\begin{bmatrix} Y_k & \tilde Z_k \end{bmatrix}$ are two orthogonal matrices, $U_k$ being a basis of the range space of $\hat A_k$, while $\tilde Z_k$ is a base of its kernel.
Finally, $L_k$ is a lower triangular matrix with strictly nonzero diagonal.
If $\hat A_k$ is full-row rank, $U_k$ is the identity and $V_k$ is empty.
Using the COD,
\begin{align}
	& \hat A_{k+1}^\dagger = Y_{k+1} L_{k+1}^{-1} U_{k+1}^\top
	\\
	& E - \hat A_{k+1}^\dagger \hat A_{k+1} = E - Y_{k+1} Y_{k+1}^\top = \tilde Z_{k+1} \tilde Z_{k+1}^\top
\end{align}%
therefore defining $Z_k$ via the recursion \eqref{eq:Z_recursion}, one gets
\begin{align}
	P_{k+1} &= Z_{k} \tilde Z_{k+1} \tilde Z_{k+1}^\top Z_k^\top M = Z_{k+1} Z_{k+1}^\top M
\end{align}%
and the proof by induction is concluded.
Moreover, using the COD of $\hat A_k$, the following holds
\begin{align}
& \bar A_k = R_k^{-1} U_k L_k Y_k^\top Z_{k-1}^\top M \\
& \bar A_k^+ = Z_{k-1} Y_k L_k^{-1} U_k^\top R_k \ .
\end{align}

\section{COD Differentiation}
\label{sec:cod}%

The COD plays a central role both from a numerical point of view for the computation of the WMPI and from a theoretical point of view for proving the existence of the coordinate transformation that embeds the configuration manifold in the hierarchical task space.

Due to the non-uniqueness of the COD, it is not possible to compute its differentiation, unless one focuses on the compact COD only.
Luckily, this is exactly what is needed to obtain the hierarchical active set and correspondent coordinate transformation.
Assume without loss of generality that ${ A \in \mathbb{R}^{m \times n} }$, with ${ m \geq n }$ and ${ r_A = \rank(A) \leq n }$.
Moreover, the columns of $A$ have been permuted so that the QR decomposition is sorted, i.e., ${ A = A^\prime \Pi }$, with $\Pi$ a permutation matrix, therefore whose entries are either zero or one.
From partitioning the QR decomposition of ${ A = Q R }$ follows:
\begin{align}
	& \begin{bmatrix} A_1 & A_2 \end{bmatrix} = \begin{bmatrix} Q_1 & Q_2 \end{bmatrix} \begin{bmatrix} R_1 & R_2 \\ O & O \end{bmatrix} \\
	& \dot A_1 = \dot Q_1 R_1 + Q_1 \dot R_1 \\
	& Q^\top \dot A_1 R_1^{-1} = \begin{bmatrix} Q_1^\top \dot Q_1 \\ Q_2^\top \dot Q_1 \end{bmatrix} + \begin{bmatrix} \dot R_1 R_1^{-1} \\ O \end{bmatrix}
\end{align}%
with ${ R_1 \in \mathbb{R}^{r_A \times r_A} }$ and the others of obvious dimensions.
Since $Q$ is orthogonal, ${ Q_1^\top \dot Q_1 }$ is skew-symmetric, while ${\dot R_1 R_1^{-1}}$ is upper-triangular.
Therefore, one can reconstruct ${ Q_1^\top \dot Q_1 }$ from the lower-triangular part of the first $r_A$ rows of ${ Q^\top \dot A_1 R_1^{-1} }$, while the last $m-r_A$ rows are exactly ${ Q_2^\top \dot Q_1 }$.
By stacking the two blocks obtained in this way and pre-multiplying them by $Q$, one obtains $\dot Q_1$.

The second step of the procedure starts by using Givens rotations to zero the entries in $R_2$ via an orthogonal matrix $U$, so that ${ R U = P }$ and ${ \begin{bmatrix} R_1 & R_2 \end{bmatrix} U = \begin{bmatrix} P_1 & O \end{bmatrix} }$
\begin{align}
	& A = Q P U^\top = Q_1 P_1 U_1^\top \\
	\begin{split}
	& P_1^{-1} Q_1^\top \dot A U = \big( P_1^{-1} Q_1^\top \dot Q_1 P_1 + P_1^{-1} \dot P_1 \big) \begin{bmatrix} E & O \end{bmatrix} \\
	& \phantom{P_1^{-1} Q_1^\top \dot A U = } + \begin{bmatrix} \dot U_1^\top U_1 & \dot U_1^\top U_2 \end{bmatrix} \\
	\end{split}
	\label{eq:dCOD}%
\end{align}%
Since ${ Q_1^\top \dot Q_1 }$ was computed in the previous step, while as before ${ \dot U_1^\top U_1 }$ is skew-symmetric and ${ P_1^{-1} \dot P_1 }$ is upper-triangular, one can reconstruct ${ \dot U_1^\top U_1 }$ from the lower-triangular part of the first $r_A$ columns of ${ P_1^{-1} Q_1^\top \dot A U }$, while the last $n-r_A$ are exactly ${ \dot U_1^\top U_2 }$.
By stacking ${ U_1^\top \dot U_1 }$, ${ U_2^\top \dot U_1 }$ and pre-multiplying them by $U$, one obtains $\dot U_1$.
Finally, since ${ P_1^{-1} \dot P_1 }$ is the only term left unknown in \eqref{eq:dCOD}, one can isolate it and pre-multiplying by $P_1$ to obtain $\dot P_1$.

\section*{Acknowledgment}
The author would like to thank Mikael Norrlof and Giacomo Spampinato for their suggestions on the draft.

\bibliographystyle{IEEEtran}
\bibliography{root}
\balance

\end{document}